\documentclass{svproc}
\usepackage{mathtools}
\usepackage{amssymb}
\usepackage{url}

\begin{document}
\mainmatter              % start of a contribution
\title{Limitations of the Lipschitz constant as a defense against adversarial examples}
\titlerunning{Limitations of the Lipschitz constant}  % abbreviated title (for running head)
%                                     also used for the TOC unless
%                                     \toctitle is used
%
\author{Todd Huster \and Cho-Yu Jason Chiang \and 
Ritu Chadha }
\authorrunning{Todd Huster et al.} % abbreviated author list (for running head)
%
%%%% list of authors for the TOC (use if author list has to be modified)
\tocauthor{Todd Huster, Cho-Yu Jason Chiang, Ritu Chadha}
\institute{Perspecta Labs, Basking Ridge, NJ 07920, USA.\\
\email{thuster@perspectalabs.com}}

\maketitle              % typeset the title of the contribution

\begin{abstract}

Several recent papers have discussed utilizing Lipschitz constants to limit the susceptibility of neural networks to adversarial examples. We analyze recently proposed methods for computing the Lipschitz constant. We show that the Lipschitz constant may indeed enable adversarially robust neural networks. However, the methods currently employed for computing it suffer from theoretical and practical limitations. We argue that addressing this shortcoming is a promising direction for future research into certified adversarial defenses.

\keywords{adversarial examples, Lipschitz constant}

\end{abstract}

\section{Introduction}

Machine learning models, such as deep neural networks (DNNs), have been remarkably successful in performing many tasks~\cite{Collobert2011NaturalLP}~\cite{Hinton2012DeepNN}~\cite{NIPS2012_4824}. However, it has been shown that they fail catastrophically when very small distortions are added to normal data examples~\cite{Goodfellow2014ExplainingAH}~\cite{Szegedy2013IntriguingPO}. These \textit{adversarial examples} are easy to produce~\cite{Goodfellow2014ExplainingAH}, transfer from one model to another~\cite{Papernot2017PracticalBA}~\cite{Tramr2017TheSO}, and are very hard to detect~\cite{Carlini2017AdversarialEA}. 

Many methods have been proposed to address this problem, but most have been quickly overcome by new attacks~\cite{Athalye2018ObfuscatedGG}~\cite{Carlini2017TowardsET}. This cycle has happened regularly enough that the burden of proof is on the defender that her or his defense will hold up against future attacks. One promising approach to meet this burden is to compute and optimize a \textit{certificate}: a guarantee that no attack of a certain magnitude can change the classifier's decision for a large majority of examples. 

In order to provide such a guarantee, one must be able to bound the possible outputs for a region of input space. This can be done for the region around a specific input~\cite{Kolter2017ProvableDA} or by globally bounding the sensitivity of the function to shifts on the input, i.e., the function's Lipschitz constant~\cite{Raghunathan2018CertifiedDA}~\cite{Tsuzuku2018LipschitzMarginTS}. Once the output is bounded for a given input region, one can check whether the class changes. If not, there is no adversarial example in the region. If the class does change, the model can alert the user or safety mechanisms to the possibility of manipulation. 

We argue in this paper that despite the achievements reported in~\cite{Raghunathan2018CertifiedDA}, Lipschitz-based approaches suffer from some representational limitations that may prevent them from achieving higher levels of performance and being applicable to more complicated problems. We suggest that directly addressing these limitations may lead to further gains in robustness. 

This paper is organized as follows: Section \ref{sec:Lipschitz} defines the Lipschitz constant and shows that classifiers with strong Lipschitz-based guarantees exist. Section \ref{sec:atomic} describes a simple method for computing a Lipschitz constant for deep neural networks, while Section \ref{sec:alim} presents experimental and theoretical limitations for this method. Section \ref{sec:sdp} describes an alternative method for computing a Lipschitz constant and presents some of its limitations. Finally, Section \ref{sec:concl} presents conclusions and a long term goal for future research.

\section{Lipschitz Bounds}\label{sec:Lipschitz}
We now define the Lipschitz constant referenced throughout this paper. 
\begin{definition}\label{def:Lipschitz}
Let a function $f$ be called $k$-Lipschitz continuous if 

\begin{eqnarray}
\forall x_{1}, x_{2} \in X: d_{Y}(f(x_{1}),f(x_{2})) \leq k d_{X} (x_{1},x_{2})
\end{eqnarray}

where $d_{X}$ and $d_{Y}$ are the metrics associated with vector spaces $X$ and $Y$, respectively.
\end{definition}

Loosely speaking, a Lipschitz constant $k$ is a bound on the slope of $f$: if the input changes by $ \epsilon $, the output changes by at most $k \epsilon $. If there is no value $\hat{k}$ where $f$ is $\hat{k}$-Lipschitz continuous and $\hat{k}<k$, then we say $k$ is the minimal Lipschitz constant. In this paper, we restrict our analysis to Minkowski $L_p$ spaces with distance metric $\|\cdot\|_p$.
We now show that global Lipschitz constants can in principle be used to provide certificates far exceeding the current state-of-the-art, and thus are worthy of further development. 
\begin{proposition}\label{prop:lip}
Let $\mathcal{D}$ be a dataset $\mathcal{D}=\big\{(x_i,y_i) \mid i=1,...,m, x_i \in \mathbb{R}^d, y_i \in \{-1,1\}\big\}$ where $x_i \neq x_j$ for $y_i \neq y_j$.  Let $c$ be a positive scalar such that
\begin{equation}
  \forall i,j: y_i \neq y_j \rightarrow ||x_i-x_j||_p>c
\end{equation}
for $p \geq1$. There exists a $\frac{2}{c}$-Lipschitz function $f: X \rightarrow \mathbb{R}$ where $\forall i: sign(f(x_i + \delta)) = y_i$ for $||\delta||_p < \frac{c}{2}$.
\end{proposition}
\begin{proof}
We relegate the full proof to appendix \ref{proof:prop}, but we define a function meeting the criteria of the proposition that can be constructed for any dataset:
\begin{equation}
f(x) =
  \begin{cases}
                                   1-\frac{2}{c}||x-x^+||_p & \text{if $||x-x^+||_p< \frac{c}{2}$} \\
                                   -1+\frac{2}{c}||x-x^-||_p & \text{if $||x-x^-||_p< \frac{c}{2}$} \\
  0 & \text{otherwise}
  \end{cases}
\end{equation}
where $x^+$ and $x^-$ are the closest vectors to $x$ in $\mathcal{D}$ with $y=1$ and $y=-1$, respectively.

\end{proof}

The function $f$ described above shows that the Lipschitz method can be used to provide a robustness guarantee against any perturbation of magnitude less than $\frac{c}{2}$. This can be extended to a multi-class setting in a straightforward manner by using a set of one vs. all classifiers.  Table \ref{tab:dist} shows the distance to the closest out-of-class example for the 95th percentile of samples; i.e., 95\% of samples are at least $c$ away from the nearest neighbor of a different class. Proposition \ref{prop:lip} implies the existence of a classifier that is provably robust for 95\% of samples against perturbations of magnitude $\frac{c}{2}$. This bound would far exceed the certifications offered by current methods, i.e.,~\cite{Kolter2017ProvableDA}~\cite{Raghunathan2018CertifiedDA}~\cite{Tsuzuku2018LipschitzMarginTS}, and even the (non-certified) adversarial performance of~\cite{Madry2017TowardsDL}.
\begin{table}
\caption{Distances to closest out-of-class example, 95th percentile.}
\begin{center}\label{tab:dist}
\begin{tabular}{c c c}
\hline
Metric   &     MNIST& CIFAR-10 \\
\hline
$L_1$  &     29.4& 170.8\\
$L_2$  &   4.06& 4.58\\
$L_{\infty}$  &   0.980& 0.392\\
\hline
\end{tabular}
\end{center}
\end{table}

It is important to note that the existence of a $\frac{c}{2}$-Lipschitz function in Proposition \ref{prop:lip} does not say anything about how easy it is to learn such a function from examples that generalizes to new ones. Indeed, the function described in the proof is likely to generalize poorly. However, we argue that current methods for optimizing the Lipschitz constant of a neural network suffer much more from \textit{underfitting} than \textit{overfitting}: training and validation certificates tend to be similar, and adding model capacity and training iterations do not appear to materially improve the training certificates. This suggests that we need more powerful models. The remainder of this paper is focused on how one might go about developing more powerful models.
\section{Atomic Lipschitz Constants}\label{sec:atomic}

The simplest method for constructing a Lipschitz constant for a neural network composes the Lipschitz constants of atomic components. If $f_1$ and $f_2$ are $k_1$- and $k_2$-Lipschitz continuous functions, respectively, and $f(x) = f_2(f_1(x))$, then $f$ is $k$-Lipschitz continuous where $k=k_1k_2$. Applying this recursively provides a bound for an arbitrary neural network.

For many components, we can compute the minimal Lipschitz constant exactly. For linear operators, $l_{W,b}(x) = Wx+b$, the minimal Lipschitz constant is given by the matrix norm of $W$ induced by $L_p$:

\begin{eqnarray}
\|W\|_p=\sup_{x\neq0}\frac{\|Wx\|_p}{\|x\|_p}
\end{eqnarray}

For $p=\infty$, this is equivalent to the largest magnitude row of $W$:

\begin{eqnarray}\label{def:inf_norm}
\|W\|_\infty=\max_{w_i \in W}\|w_i\|_1
\end{eqnarray}

The $L_2$ norm of $W$ is known as its \textit{spectral norm} and is equivalent to its largest singular value.  The element-wise ReLU function $ReLU(x)=max(x,0)$ has a Lipschitz constant of 1 regardless of the choice of $p$.  Therefore, for a neural network $f$ composed of $n$ linear operators $l_{W_1,b_2},...,l_{W_n,b_n}, $ and ReLUs, a Lipschitz constant $k$ is provided by

\begin{eqnarray}\label{eq:k}
k = \prod_{i=1}^{n}\|W_i\|_p
\end{eqnarray}

Several recent papers have utilized this concept or an extension of it to additional layer types. ~\cite{Szegedy2013IntriguingPO} uses it to analyze the theoretical sensitivity of deep neural networks. ~\cite{Ciss2017ParsevalNI} and~\cite{Qian2018L2NonexpansiveNN} enforce constraints on the singular values of matrices as a way of increasing robustness to existing attacks. Finally,~\cite{Tsuzuku2018LipschitzMarginTS} penalizes the spectral norms of matrices and uses equation \ref{eq:k} to compute a Lipschitz constant for the network.

\section{Limitations of Atomic Lipschitz Constants}\label{sec:alim}
One might surmise that this approach can solve the problem of adversarial examples: compose enough layers together with the right balance of objectives, overcoming whatever optimization difficulties arise, and one can train classifiers with high accuracy, guaranteed low variability, and improved robustness to attacks. Unfortunately, this does not turn out to be the case, as we will show first experimentally and then theoretically. 

\subsection{Experimental Limitations}

First, we can observe the limits of this technique in a shallow setting. We train a two layer fully connected neural network with 500 hidden units $f=l_{W_2,b_2} \circ ReLU \circ l_{W_1,b_1}$ on the MNIST dataset. We penalize $\|W_1\|_p\|W_2\|_p$ with weight $\lambda_p$. We denote the score for class $i$ as $f_i(x)$ and the computed Lipschitz constant of the difference between $f_i(x)$ and $f_j(x)$ as $k_{ij}$. We certify the network for example $x$ with correct class $i$ against a perturbation of magnitude $\epsilon$ by verifying that $f_i(x) - f_j(x) - k_{ij}\epsilon > 0$ for $i \neq j$.

Figures \ref{fig:lil2} (a) and (b) show results for $L_{\infty}$ and $L_2$, respectively. In both cases, adding a penalty provides a larger region of certified robustness, but increasing the penalty hurts performance on unperturbed data and eventually ceases to improve the certified region. This was true for both test and training (not shown) data. This level of certification is considerably weaker than our theoretical limit from Proposition \ref{prop:lip}. 

There also does not appear to be much certification benefit to adding more layers. We extended the methodology to multi-layer networks and show the results in figures \ref{fig:lil2} (c) and (d). Using the $\lambda_{\infty}$ penalty proved difficult to optimize for deeper networks. The $\lambda_{2}$ penalty was more successful, but only saw a mild improvement over the shallow model. The results in (d) also compare favorably to those of~\cite{Tsuzuku2018LipschitzMarginTS}, which uses a 4 layer convolutional network. 

\begin{figure}
\centering
\includegraphics[scale=0.37]{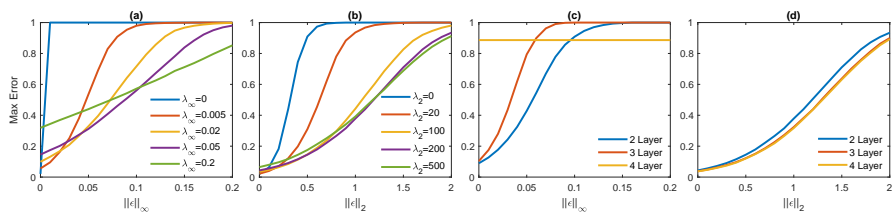}
\caption{Experimental results from atomic Lipschitz penalties. On the left, the $L_{\infty}$ norm is used for both the perturbation and the penalty, while on the right, $L_2$ is used}\label{fig:lil2}
\end{figure}

\subsection{Theoretical Limitations}
We now consider the set of neural networks with a given atomic Lipschitz bound and the functions it can compute. This set of functions is important because it limits how well a neural network can split a dataset with particular margins, and thus how strong the certificate can be. 

\begin{definition}\label{def:nn}
Let $\mathcal{A}_k^p$ be the set of neural networks with an atomic Lipschitz bound of k in $L_p$ space:

\begin{equation}
\mathcal{A}_k^p \triangleq \Big\{l_{W_n,b_n} \circ \dots \circ ReLU \circ l_{W_1,b_1} \mid \prod_i\|W_i\|_p \leq k , n \geq 2 \Big\}
\end{equation}
\end{definition} 

We focus our analysis here on $L_{\infty}$ space. To show the limitations of $\mathcal{A}_k^{\infty}$, consider the simple 1-Lipschitz function $f(x) = |x|$. Expressing $f$ with ReLU's and linear units is simple exercise, shown in figure \ref{fig:abs}. However, since 
\begin{equation}
\big\|\begin{bmatrix} 1 & -1 \end{bmatrix}\big\|_{\infty}\Bigg\|\begin{bmatrix} 1 \\ 1 \end{bmatrix}\Bigg\|_{\infty}=2,
\end{equation}
the neural network in figure \ref{fig:abs} is a member of $\mathcal{A}_2^{\infty}$, but not $\mathcal{A}_1^{\infty}$. This is only one possible implementation of $|x|$, but as we will show, the atomic component method cannot express this function with a Lipschitz bound lower than 2, and the situation gets worse as more non-linear variations are added.  

\begin{figure}
\centering
\includegraphics[scale=0.6]{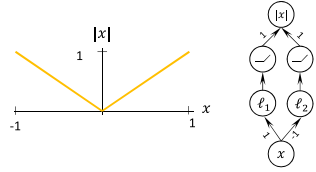}
\caption{The absolute value function (left) and a neural network that implements it (right)}\label{fig:abs}
\end{figure}

We now provide two definitions that will help delineate the functions that the neural networks in $\mathcal{A}_k^{\infty}$ can compute. 

\begin{definition}\label{def:var}
For a function $f: \mathbb{R} \rightarrow \mathbb{R}$, let the total variation be defined as 
\begin{equation}
V_{a}^{b}(f) \triangleq \sup_{T \in \mathcal{T}}\sum_{t_i \in T} \lvert f(t_i)-f(t_{i-1})\rvert
\end{equation}
where $\mathcal{T}$ is the set of partitions of the interval $[a,b]$.
\end{definition}

The total variation captures how much a function changes over its entire domain, which we will use on the gradients of neural networks. $V_{-\infty}^{\infty}$ is finite for neural network gradients, as the gradient only changes when a ReLU switches states, and this can only happen a finite number of times for finite networks. Clearly, for the slope of the absolute value function, this quantity is 2: the slope changes from -1 to 1 at $x=0$. 

\begin{definition}\label{def:intrinsic}
For a function $f: \mathbb{R} \rightarrow \mathbb{R}$, define a quantity

\begin{equation}
I(f) \triangleq V_{-\infty}^{\infty}(f) + \lvert f(\infty)\rvert + \lvert f(-\infty)\rvert
\end{equation}
and call it the intrinsic variability of $f$.
\end{definition}

As we will show, the intrinsic variability is a quantity that is nonexpansive under the ReLU operation. The intrinsic variability the slope of the absolute value function is 4: we add the magnitude of the slopes at the extreme points, 1 in each case, to the total variation of 2. We now begin a set of proofs to show that $\mathcal{A}_k^{\infty}$ is limited in the functions it can approximate. This limit does not come from the Lipschitz constant of a function $f$, but by the intrinsic variability of its derivative, $f'$. 

\begin{lemma}\label{lemma:lin_i}
For a linear combination of functions $f(x) = \sum_i w_i f_i(x)$,
\begin{equation}
I(f') \leq \sum_i |w_i|I(f_i').
\end{equation}
\end{lemma}
\begin{proof}
Proof is relegated to appendix \ref{proof:lin}
\end{proof} 
\begin{definition}
Let a function $f:\mathbb{R} \rightarrow \mathbb{R} $ be called eventually constant if
\begin{equation}
\exists t_- \in \mathbb{R}, f'(t) = f'(t_-), t \leq t_-
\end{equation}
\begin{equation}
\exists t_+ \in \mathbb{R}, f'(t) = f'(t_+), t \geq t_+
\end{equation}
\end{definition}

\begin{lemma}\label{lemma:relu_i}
Let $f(t)$ be a function where $f'(t)$ is eventually constant. For the ReLU activation function $g(t) = max(f(t),0)$,
\begin{equation}
I(g') \leq I(f')
\end{equation}
\end{lemma}

\begin{proof}
Proof is relegated to appendix \ref{proof:relu}
\end{proof} 
\begin{theorem}\label{theorem:inf_tv}

Let $f \in \mathcal{A}_k^{\infty}$ be a scalar-valued function $f: \mathbb{R}^d \rightarrow \mathbb{R}$. 

Let $h_{W_0,b_0} = f \circ l_{W_0,b_0}$ where $W_0 \in \mathbb{R}^{d \times 1}$, $b_0 \in \mathbb{R}^d$ and $\|W_0\|_{\infty} = 1$. For any selection of $W_0$ and $b_0$, 
\begin{equation}
I(h_{W_0,b_0}') \leq 2k.
\end{equation}
\end{theorem}

\begin{proof}
Proof is relegated to appendix \ref{proof:atomic}
\end{proof} 

A function in $\mathcal{A}_k^{\infty}$ has a hard limit on the intrinsic variability of its slope along a line through its input space. If we try to learn the absolute value function while penalizing the bound $k$, we will inevitably end up with training objectives that are in direct competition with one another. One can imagine more difficult cases where there is some oscillation in the data manifold and the bounds deteriorate further: for instance $sin(x)$ is also 1-Lipschitz, but can only be approximated with arbitrarily small error by a member of $\mathcal{A}_\infty^{\infty}$. While this limit is specific to $\mathcal{A}_k^{\infty}$, since $\|W\|_2 \leq \|W\|_{\infty}$, it also provides a limit to $\mathcal{A}_k^{2}$.

\section{Paired-layer Lipschitz Constants and Their Limitations}\label{sec:sdp}

We have shown the limitations of the atomic bounding method both experimentally and theoretically, so naturally we look for other approaches to bounding the Lipschitz constant of neural network layers. A fairly successful approach was given by~\cite{Raghunathan2018CertifiedDA}. ~\cite{Raghunathan2018CertifiedDA} presents a method for bounding a fully connected neural network with one hidden layer and ReLU activations, which yielded impressive performance on the MNIST dataset. This approach optimizes the weights of the two layers in concert, so we call it the \textit{paired-layer} approach. The paper does not attempt to extend the method to deeper neural networks, but it can be done in a relatively straightforward fashion.

\subsection{Certifying a Two-layer Neural Network}
Ignoring biases for notational convenience, a two-layer neural network with weights $W_1$ and $W_2$ can be expressed 
\begin{equation}
f(x) = W_2 diag(s) W_1 x
\end{equation}
where $s = W_1 x > 0$. We consider a single output, although extending to a multi-class setting is straightforward. If $s$ were fixed, such a network would be linear with Lipschitz constant $\|W_2 diag(s) W_1\|_p$. ~\cite{Raghunathan2018CertifiedDA} accounts for a changeable $s$ by finding the assignment of $s$ that maximizes the $L_{\infty}$ Lipschitz constant and using this as a bound for the real Lipschitz constant:
\begin{equation}\label{eq:qpbound}
k \leq \max_{s \in \{0,1\}^d} \|W_2 diag(s) W_1\|_{\infty}
\end{equation}

They convert this problem to a mixed integer quadratic program and bound it in a tractable and differential manner using semi-definite programming, the details of which are explained in ~\cite{Raghunathan2018CertifiedDA}. We can add a penalty on this quantity to the objective function to find a model with relatively high accuracy and low Lipschitz constant. We did not have access to the training procedure developed by ~\cite{Raghunathan2018CertifiedDA}, but we were able to closely replicate their results on MNIST and compare them to the atomic bounding approach, shown in figure \ref{fig:sdp} (a).

\begin{figure}
\centering
\includegraphics[scale=0.6]{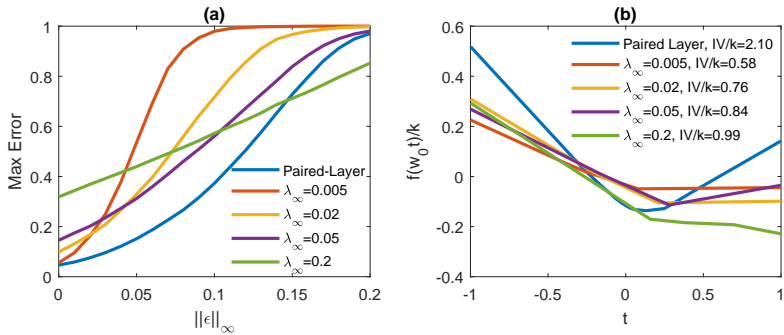}
\caption{(a) Results comparing penalizing the atomic Lipschitz bound and the paired-layer bound (b) Neural network outputs along the line $w_0 t$ and their intrinsic varibilities. Values are scaled by the given Lipschitz constant}\label{fig:sdp}
\end{figure}

%We also extend their approach to additional layers 

\subsection{Theoretical Benefits and Limitations of Paired-layer Approach}

Figure \ref{fig:sdp} shows that there are practical benefits to the paired-layer approach, and we can also show a corresponding increase in expressive power. Similar to $\mathcal{A}_k^p$, we define a set of neural networks $\mathcal{M}_k$, although we will restrict the definition to 2 layer networks in $L_{\infty}$ space:

\begin{definition}\label{def:sdpnn}
Let $\mathcal{M}_k$ be the set of two-layer neural networks with a paired-layer Lipschitz bound of k in $L_{\infty}$ space:

\begin{equation}
\mathcal{M}_k \triangleq \Big\{l_{W_2,a_2} \circ ReLU \circ l_{W_1,a_1} \mid \max_{s \in \{0,1\}^d} \|W_2 diag(s) W_1\|_{\infty} \leq k \Big\}
\end{equation}
\end{definition} 

$\mathcal{M}_k$ can express functions that $\mathcal{A}_k^{\infty}$ cannot. For example, we can apply the paired-layer method to the neural network in figure \ref{fig:abs} by enumerating the different cases. In this case the bound is tight, meaning that the neural network is in $\mathcal{M}_1$. From Theorem \ref{theorem:inf_tv}, we know that this function cannot be expressed by any member of $\mathcal{A}_1^{\infty}$. It is easy to see that any two layer neural network in $\mathcal{A}_k^{\infty}$ is also in $\mathcal{M}_k$, so we can say confidently that the paired-layer bounds are tighter than atomic bounds. 

This additional expressiveness is not merely academic. Figure \ref{fig:sdp} (b) shows the output of the networks from (a) along a particular line in input space, scaled by the given Lipschitz bound. The function learned by the paired-layer method does in fact exhibit an intrinsic variability larger than $2k$, meaning that function cannot be represented by a network in $\mathcal{A}_k^{\infty}$. This suggests that the gains in performance may be coming from the increased expressiveness of the model family.

%\begin{figure}
%\centering
%\includegraphics[scale=0.5]{IV}
%\caption{Neural network outputs along the line $w_0 t$ and their intrinsic varibilities. Values are scaled by the given Lipschitz constant}\label{fig:iv}
%\end{figure}

%\begin{equation}
%\big\|\begin{bmatrix} 1 & 1 \end{bmatrix}\big\|_{\infty}\Bigg\|\begin{bmatrix} 1 \\ 1 \end{bmatrix}\Bigg\|_{\infty}=2,
%\end{equation}

%\begin{equation}
%\Bigg\| \begin{bmatrix} 1 & -1 \end{bmatrix}\begin{bmatrix} 1 & 0 \\ 0 & 0 \end{bmatrix}\begin{bmatrix} 1 \\ 1 \end{bmatrix}\Bigg\|_{\infty}=1,\Bigg\| \begin{bmatrix} 1 & -1 \end{bmatrix}\begin{bmatrix} 1 & 0 \\ 0 & 1 \end{bmatrix}\begin{bmatrix} 1 \\ 1 \end{bmatrix}\Bigg\|_{\infty}=0,
%\end{equation}
%\begin{equation}
%\Bigg\| \begin{bmatrix} 1 & -1 \end{bmatrix}\begin{bmatrix} 0 & 0 \\ 0 & 1 \end{bmatrix}\begin{bmatrix} 1 \\ 1 \end{bmatrix}\Bigg\|_{\infty}=1,\Bigg\| \begin{bmatrix} 1 & -1 \end{bmatrix}\begin{bmatrix} 0 & 0 \\ 0 & 0 \end{bmatrix}\begin{bmatrix} 1 \\ 1 \end{bmatrix}\Bigg\|_{\infty}=0,
%\end{equation}

It is still easy to construct functions for which the paired-layer bounds are loose, however. Figure \ref{fig:f2} shows a 1-Lipschitz function and a corresponding neural network that is only in $\mathcal{M}_2$. The problem arises from the fact that the two hidden units cannot both be on, but the quadratic programming problem in equation \ref{eq:qpbound} implies that they can. For a 1-D problem, the bound essentially adds up the magnitudes of the paths with positive weights and the paths with negative weights and takes the maximum. A higher dimensional problem can be reduced to a 1-D problem by considering arbitrary lines through the input space. 

\begin{figure}
\centering
\includegraphics[scale=0.6]{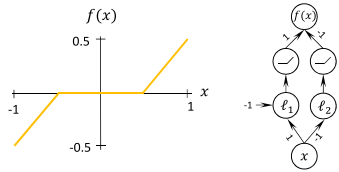}
\caption{A 1-Lipchitz function (left) and a neural network that implements it (right)}\label{fig:f2}
\end{figure}

The expressive limitations of $\mathcal{M}_k$ are apparent when we consider its components. Any neural network in $\mathcal{M}_k$ is a sum of combinations of the four basic forms in figure \ref{fig:sdp_comp}, with various biases and slopes. The sum of the slope magnitudes from the positive paths can be no greater than $k$, and likewise for the negative paths. Each form has a characteristic way of affecting the slope at the extremes and changing the slope.  For instance form (a) adds a positive slope at $+\infty$ as well as a positive change in $f'$. From here we can see that there is still a connection between the total variation and extreme values of $f'$ and the bound $k$. While the paired-layer bounds are better than the atomic ones, they still become arbitrarily bad for e.g., oscillating functions.

\begin{figure}
\centering
\includegraphics[scale=0.6]{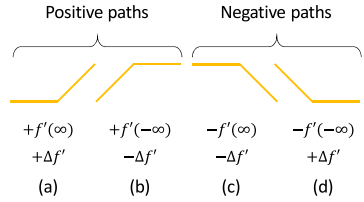}
\caption{The four forms of components of a two layer neural network, and their distinguishing characteristics}\label{fig:sdp_comp}
\end{figure}

\section{Conclusions}\label{sec:concl}

We have presented a case that existing methods for computing a Lipschitz constant of a neural network suffer from representational limitations that may be preventing them from considerably stronger robustness guarantees against adversarial examples. Addressing these limitations should enable models that can, at a minimum, exhibit strong guarantees for training data and hopefully extend these to out-of-sample data. Ideally, we envision \textit{universal Lipschitz networks}: a family of neural networks that can represent an arbitrary k-Lipschitz function with a tight bound. The development of such a family of models and methods for optimizing them carries the potential of extensive gains in adversarial robustness.

\paragraph{Acknowledgement:} This research was partially sponsored by the U.S. Army Research Laboratory and was accomplished under Cooperative Agreement Number W911NF-13-2-0045 (ARL Cyber Security CRA). The views and conclusions contained in this document are those of the authors and should not be interpreted as representing the official policies, either expressed or implied, of the Army Research Laboratory or the U.S. Government. The U.S. Government is authorized to reproduce and distribute reprints for Government purposes notwithstanding any copyright notation here on.

\bibliographystyle{spmpsci}
\bibliography{refs}{}
\appendix

\section{Proofs}
\subsection{Proof of Proposition \ref{prop:lip}}\label{proof:prop}
\begin{proof}
Consider the function 
\begin{equation}
f(x) =
  \begin{cases}
                                   1-\frac{2}{c}||x-x^+||_p & \text{if $||x-x^+||_p< \frac{c}{2}$} \\
                                   -1+\frac{2}{c}||x-x^-||_p & \text{if $||x-x^-||_p< \frac{c}{2}$} \\
  0 & \text{otherwise}
  \end{cases}
\end{equation}
where $x^+$ and $x^-$ are the closest vectors to $x$ in $\mathcal{D}$ with $y=1$ and $y=-1$, respectively. Since $||x^+-x^-||_p > c$, the conditions are mutually exclusive. When $y_i = 1$ and $||\delta||_p < \frac{c}{2}$, 
\begin{equation} 
f(x_i + \delta) = 1-\frac{2}{c} ||x_i^+-x^+ + \delta||_p \geq 1-\frac{2}{c}||\delta||_p \geq 0.
\end{equation}
The inverse is true for $y_i = -1$, therefore $sign(f(x_i + \delta)) = y_i$ holds for all $i$. $f$ is continuous at the non-differentiable boundaries between the piecewise conditions of $f$ and the selections of $x^+$ and $x^-$. Therefore, it suffices to show that each continuously differentiable piece is $\frac{2}{c}$-Lipschitz. Using Definition \ref{def:Lipschitz}, we must show 
\begin{equation} 
|f(x)-f(x+\delta)| \leq \frac{2}{c}||\delta||_p.
\end{equation}
For the first condition of $f$ with a fixed $x^+$, we get
\begin{equation} 
\Bigg|1-\frac{2}{c}||x-x^+||_p - \Bigg(1-\frac{2}{c}||x+\delta-x^+||_p\Bigg) \Bigg| \leq \frac{2}{c}||\delta||_p
\end{equation}
\begin{equation} 
\frac{2}{c}\Big|||x-x^+||_p - ||x+\delta-x^+||_p \Big| \leq \frac{2}{c}||\delta||_p,
\end{equation}
which holds for $p \geq 1$ due to the Minkowski inequality. The same holds for the second condition. Since the third condition is constant, $f(x)$ must be $\frac{2}{c}$-Lipschitz and the proof is complete.  \qed
\end{proof}
\subsection{Proof of Lemma \ref{lemma:lin_i}}\label{proof:lin}
\begin{proof}

Using the chain rule, we get
\begin{equation}
f'(t) = \sum_i w_i f_i'(t).
\end{equation}

The triangle inequality gives us the following two inequalities
\begin{equation}\label{eq:i1}
\lvert f'(t)\rvert  = \Big\lvert \sum_i w_i f_i'(t) \Big\rvert \leq \sum_i \lvert w_i f_i'(t)\rvert = \sum_i \lvert w_i \rvert \lvert f_i'(t)\rvert
\end{equation}
\begin{equation}\label{eq:i2}
\lvert f'(t_i)-f'(t_{i-1})\rvert \leq \sum_i \lvert w_i f_i'(t_i)-w_i f_i'(t_{i-1})\rvert = \sum_i \lvert w_i\rvert \lvert f_i'(t_i)-f_i'(t_{i-1})\rvert
\end{equation}

Let $T_{f'}$ be a maximal partition for $V_{\infty}^{\infty}(f')$, giving us

\begin{equation}
I(f') = \sum_{t_i \in T_{f'}}\lvert f'(t_i)-f'(t_{i-1})\rvert + \lvert f'(\infty)\rvert + \lvert f'(-\infty)\rvert
\end{equation}

We complete the proof by substituting with (\ref{eq:i1}) and (\ref{eq:i2}) and reordering the terms :

\begin{equation}
I(f') \leq \sum_i \lvert w_i\rvert \Big(   \sum_{t_i \in T_f} \lvert f_i'(t_i)-f_i'(t_{i-1})\rvert +  \lvert f_i'(\infty)\rvert +   \lvert f_i'(-\infty)\rvert \Big) = \sum_i |w_i|I(f_i').
\end{equation}

%\begin{equation}\label{eq:i1}
%I(\nabla_f) \leq \sum_i |v_i| \Big(   \sum_{t_i \in T_f} \|F_i(t)-F_i(t_{i-1})\|_1 +  \|F_i(\infty)\|_1 +   \|F_i(-\infty)\|_1\Big)
%\end{equation}
%By definitions \ref{def:var} and \ref{def:intrinsic}, we get

%\begin{equation}\label{eq:i2}
%\sum_{t_i \in T_f}\|F_i(t)-F_i(t_{i-1})\|_1 + \|F_i(\infty)\|_1 + \|F_i(-\infty)\|_1 \leq I(\nabla_{f_i}) 
%\end{equation}

%We complete the proof by substituting the inequality in \ref{eq:i2} for each $i$ in \ref{eq:i1}. Since all terms are nonnegative, we get

%\begin{equation}
%\sum_i |v_i| \Big(   \sum_{t_i \in T_f} \|F_i(t)-F_i(t_{i-1})\|_1 +  \|F_i(\infty)\|_1 +   \|F_i(-\infty)\|_1\Big) \leq \sum_i |v_i|I(\nabla_{f_i})
%\end{equation}
\qed
\end{proof}
\subsection{Proof of Lemma \ref{lemma:relu_i}}\label{proof:relu}
\begin{proof}

Let $[t_-,t_+]$ be an interval outside of which $f'(t)$ is constant. Assume that $f'(t) > 0$ for $t \in [t_-,t_+])$. In this case, 
\begin{equation}
V_{t_-}^{t_+}(g') = V_{-\infty}^{\infty}(f'). 
\end{equation}

If $f'(-\infty) > 0$ then at some point $t<t_-$, $f(t)=0$ and $g'$ transitions from $f'(-\infty)$ to 0. Otherwise for $t<t_-$, $g'(t)=f'(t)$. Therefore, 

\begin{equation}
V_{-\infty}^{t_-}(g') + \lvert g'(-\infty)\rvert = \lvert f'(-\infty)\rvert
\end{equation}

Similarly,
\begin{equation}
V_{t_+}^{\infty}(g') + \lvert g'(\infty)\rvert = \lvert f'(\infty)\rvert
\end{equation}

Putting the different intervals together, we get
\begin{equation}
I(g') = V_{-\infty}^{t_-}(g') + \lvert g'(-\infty)\rvert + V_{t_+}^{\infty}(g') + \lvert g'(\infty)\rvert + V_{t_-}^{t_+}(g') 
\end{equation}
\begin{equation}
I(g') = \lvert f'(-\infty)\rvert + \lvert f'(\infty)\rvert + V_{-\infty}^{\infty}(f')
\end{equation}
\begin{equation}
I(g') \leq I(f')
\end{equation}
So the statement holds when our assumption about $f$ is met. To address cases where $f$ has negative values in $[t_-,t_+]$, consider an interval $(t_1, t_2)$ where $g(t_1)=f(t_1), g(t_2)=f(t_2), g(t) \neq f(t) for t_1 < t < t_2$. We note that $f'(t_1) < 0$ and $f'(t_2) > 0$. Since $f'$ must transition from $f'(t_1)$ to $f'(t_2)$, over $(t_1, t_2)$,
\begin{equation}
V_{t_1}^{t_2}(f') \geq \lvert f'(t_1) \rvert + \lvert f'(t_2) \rvert.
\end{equation}

Since $g'$ transitions from $f'(t_1)$ to 0 to $f'(t_2)$ over $(t_1, t_2)$ so,
\begin{equation}
V_{t_1}^{t_2}(g') = \lvert f'(t_1) \rvert + \lvert f'(t_2) \rvert.
\end{equation}

Applying this to all such intervals gives us 
\begin{equation}
V_{t_-}^{t_+}(g') \leq V_{t_-}^{t_+}(f')
\end{equation}
and therefore $I(g') \leq I(f')$

\qed
\end{proof}
\subsection{Proof of Theorem \ref{theorem:inf_tv}}\label{proof:atomic}
\begin{proof}
Combining the definition of $h_{W_0,b_0}$ with Definition \ref{def:nn}, we can see that $h_{W_0,b_0} = l_{W_n,b_n} \circ \dots \circ ReLU \circ l_{W_1,b_1} \circ l_{W_0,b_0}$ and $\prod_{i=0}^n\|W_i\|_{\infty} \leq k$. We consider the additional linear transform as the zeroth layer of a modified network. Consider unit $u$ in the zeroth layer as a function $\sigma_{0,u}(t)$. $\sigma_{0,j}'(t)$ is constant, with 
\begin{equation}
\sigma_{0,u}'(t) = \lvert w^0_{u,1} \rvert \leq 1
\end{equation}
where $w^i_{u,v}$ is element $(u,v)$ of $W_i$. Therefore
\begin{equation}
V_{-\infty}^{\infty}(\sigma_{0,u}') = 0
\end{equation}

We also have $\forall t, \lvert \sigma_{0,u}' \rvert = \lvert w^0_{u,1}\rvert$, so by Definition \ref{def:intrinsic}

\begin{equation}\label{eq:base_case}
I(\sigma_{0,u}') = 2\lvert w^0_{u,1}\rvert \leq 2.
\end{equation}

%We also have $\forall x, \|\nabla_{g_{1,j}}(x)\|_1 = \|w^1_j\|_1$, so 
%\begin{equation}\label{eq:base_case}
%I(\nabla_{g_{1,j}}) = 2\|w^1_j\|_1 \leq 2\|W_1\|_{\infty}.
%\end{equation}
%\begin{equation}
%I(\nabla_{g_{1,j}}) = V_{-\infty}^{\infty}(\nabla_{g_{1,j}}(ct+d)) + 2\|\nabla_{g_{1,j}}(x_0)\|_1 \leq 2\|W_1\|_{\infty}
%\end{equation}
We recursively define functions for each unit in layers $1$ to $n$:
\begin{equation}
g_{i,v}(t) = \sum_u w^i_{u,v} \sigma_{i-1,u}(t)
\end{equation}
\begin{equation}
\sigma_{i,u}(t) = \max(g_{i,u}(t),0)
\end{equation}

Applying Lemma \ref{lemma:relu_i} and noting that a function composed of ReLU and linear operators is eventually constant, we get 
\begin{equation}
I(\sigma_{i,u}') \leq I(g_{i,u}')
\end{equation}
Applying Lemma \ref{lemma:lin_i}, we get 
\begin{equation}
I(g_{i,v}') \leq \sum_u |w^i_{u,v}| I(\sigma_{i-1,u})
\end{equation}
Furthermore, we can say 
\begin{equation} \label{eq:recursive}
\max_v I(g_{i,v}') \leq \|W_i\|_{\infty} \max_u I(g_{i-1,u}')
\end{equation}
Finally, we conclude the proof by recursively applying (\ref{eq:recursive}) on the base case in (\ref{eq:base_case}) to yield
\begin{equation}
I(h_{W_0,b_0}') = I(g_{n,1}') \leq 2\prod_{i=1}^{n}\|W_i\|_{\infty} \leq 2k
\end{equation}
\qed
\end{proof}

\end{document}